\documentclass{article}

\PassOptionsToPackage{numbers, compress}{natbib}
\usepackage[final]{neurips_2019}

\usepackage[utf8]{inputenc} % allow utf-8 input
\usepackage[T1]{fontenc}    % use 8-bit T1 fonts
\usepackage{hyperref}       % hyperlinks
\usepackage{url}            % simple URL typesetting
\usepackage{booktabs}       % professional-quality tables
\usepackage{amsfonts}       % blackboard math symbols
\usepackage{nicefrac}       % compact symbols for 1/2, etc.
\usepackage{microtype}      % microtypography

\usepackage{microtype}
\usepackage{graphicx}
\usepackage{subfigure}
\usepackage{booktabs} % for professional tables
\usepackage{amsmath,amsfonts,amsthm}
\usepackage{multirow}
\usepackage{algorithm}
\usepackage{natbib}
\usepackage{color}
\newtheorem{thm}{Theorem}
\newtheorem{dfn}{Definition}

\newtheorem{lem}{Lemma}
\newtheorem{ex}{Example}

\newtheorem{coro}{Corollary}

\newtheorem{claim}{Claim}

\newcommand{\ptopl}[1]{\phi^{\text{top-}#1}_{\ma}}
\newcommand{\plway}[2]{\phi^{#1\text{-way}}_{#2}}
\newcommand{\pchoice}[2]{\phi^{\text{choice-}#1}_{#2}}
\newcommand{\splway}[1]{\Phi^{#1\text{-way}}}
\newcommand{\spchoice}[1]{\Phi^{\text{choice-}#1}}

\newcommand{\fek}{\mathbf{F}^k_E}

\newcommand{\hek}{\mathbf{H}^k}
\newcommand{\rank}{\text{rank}}

\newcommand\Omit[1]{}

\newcommand{\ma}{\mathcal A}

\newcommand{\ml}{\mathcal L}

\newcommand{\mcp}{\mathcal P}
\newcommand{\mm}{\mathcal M}

\newcommand{\me}{\mathcal E}

\newcommand{\topl}{{\text{top-}l}}
\newcommand{\lway}{{l\text{-way}}}
\newcommand{\kplp}{$k$-PL-$\Phi$}

\title{Learning Mixtures of Plackett-Luce Models from Structured Partial Orders}

\author{%
  Zhibing Zhao\\%\thanks{Use footnote for providing further information} \\
  Department of Computer Science\\
  Rensselaer Polytechnic Institute\\
  Troy, NY 12180 \\
  \texttt{zhaoz6@rpi.edu} \\
  \And
  Lirong Xia \\
  Department of Computer Science\\
  Rensselaer Polytechnic Institute\\
  Troy, NY 12180 \\
  \texttt{xial@cs.rpi.edu} \\
}

\begin{document}

\maketitle

\begin{abstract}
Mixtures of ranking models have been widely used for heterogeneous preferences. However, learning a mixture model is highly nontrivial, especially when the dataset consists of partial orders. In such cases, the parameter of the model may not be even identifiable. In this paper, we focus on three popular structures of partial orders: ranked top-$l_1$, $l_2$-way, and choice data over a subset of alternatives. We prove that when the dataset consists of combinations of ranked top-$l_1$ and $l_2$-way (or choice data over up to $l_2$ alternatives), mixture of $k$ Plackett-Luce models is not identifiable when $l_1+l_2\le 2k-1$ ($l_2$ is set to $1$ when there are no $l_2$-way orders). We also prove that under some combinations, including ranked top-$3$, ranked top-$2$ plus $2$-way, and choice data over up to $4$ alternatives, mixtures of two Plackett-Luce models are identifiable. Guided by our theoretical results, we propose efficient generalized method of moments (GMM) algorithms to learn mixtures of two Plackett-Luce models, which are proven consistent. Our experiments demonstrate the efficacy of our algorithms. Moreover, we show that when full rankings are available, learning from different marginal events (partial orders) provides tradeoffs between statistical efficiency and computational efficiency.
\end{abstract}

\section{Introduction}
\label{intro}

%Consider the group activity selection problem \citep{Darmann12:Group}, where we are organizing social activities for a group of participants. There are many candidate activities such as hiking, city tour, etc. Each participant can only take part in one activity and only two activities can be chosen in total. Participants have different preferences over the activities and they may submit preferences following different structures. For example, some people provide their ranked top two alternatives, while others give some pairwise comparisons. How can we decide the two activities for the group of participants?

Suppose a group of four friends want to choose one of the four restaurants $\{a_1,a_2,a_3,a_4\}$ for dinner. The first person ranks all four restaurants as $a_2\succ a_3\succ a_4\succ a_1$, where $a_2\succ a_3$ means that ``$a_2$ is strictly preferred to $a_3$''. The second person says ``$a_4$ and $a_3$ are my top two choices, among which I prefer $a_4$ to $a_3$''. The third person ranks $a_3\succ a_4\succ a_1$ but has no idea about $a_2$. The fourth person has no idea about $a_4$, and would choose $a_3$ among $\{a_1,a_2, a_3\}$. How should they aggregate their preferences to choose the best restaurant?

Similar {\em rank aggregation} problems exist in social choice, crowdsourcing~\cite{Mao13:Better,Chen13:Pairwise}, recommender systems~\cite{Candes09:Exact,Baltrunas10:Group,Keshavan10:Matrix,Negahban12:Restricted}, information retrieval \cite{Altman05:PageRank,Liu11:Learning}, etc. Rank aggregation can be cast as the following statistical parameter estimation problem: given a statistical model for rank data and the agents' preferences, the parameter of the model is estimated to make decisions. Among the most widely-applied statistical models for rank aggregation are the Plackett-Luce model~\cite{Luce59:Individual,Plackett75:Analysis} and its mixtures~\cite{Gormley08:Exploring,Gormley09:Grade,Liu11:Learning,Mollica17:Bayesian,Tkachenko16:Plackett, Mollica17:Bayesian}. In a Plackett-Luce model over a set of alternatives $\ma$, each alternative is parameterized by a strictly positive number that represents its probability to be ranked higher than other alternatives. A mixture of $k$ Plackett-Luce models, denoted by $k$-PL, combines $k$ component Plackett-Luce models via the {\em mixing coefficients} $\vec \alpha =(\alpha_1,\ldots,\alpha_k)\in {\mathbb R}_{\ge 0}^k$ with $\vec\alpha\cdot\vec 1 = 1$, such that for any $r\le k$, with probability $\alpha_r$, a data point is generated from the $r$-th Plackett-Luce component.

One critical limitation of Plackett-Luce model and its mixtures is that their sample space consists of {\em linear orders} over $\ma$. In other words, each data point must be a full ranking of all alternatives in $\ma$. However, this is rarely the case in practice, because agents are often not able to rank all alternatives due to lack of information~\cite{Pini11:Incompleteness}, as illustrated in the example in the beginning of Introduction.

In general, each rank datum is a {\em partial order}, which can be seen as a collection of pairwise comparisons among alternatives that satisfy transitivity. However, handling partial orders is more challenging than it appears. In particular, the pairwise comparisons of the same agent cannot be seen as independently generated due to transitivity.

Consequently, most previous works focused on {\em structured partial orders}, where agents' preferences share some common structures. For example, given $l\in \mathbb N$, in ranked-top-$l$ preferences~\cite{Mollica17:Bayesian,Huang11:Efficient}, agents submit a linear order over their top $l$ choices; in $l$-way preferences~\cite{Marden95:Analyzing,Hunter04:MM,Maystre15:Fast}, agents submit a linear order over a set of $l$ alternatives, which are not necessarily their top $l$ alternatives; in choice-$l$ preferences (a.k.a.~choice sets)~\cite{Train09:Discrete}, agents only specify their top choice among a set of $l$ alternatives. In particular, pairwise comparisons can be seen as $2$-way preferences or choice-$2$ preferences.

%Some works assume a fixed structure, such as ranked top-$l$, $l$-way~\citep{Hunter04:MM,Maystre15:Fast}, choice-$l$~\citep{McFadden74:Conditional}, and others~\cite{Khetan16:Data}. Approaches for generating general partial orders are mostly sampling pairwise comparisons independently~\cite{Lu14:Effective,Liu19:Learning}.

However, as far as we know, most previous works assumed that the rank data share the same structure for their algorithms and theoretical guarantees to apply. It is unclear how rank aggregation can be done effectively and efficiently from structured partial orders of different kinds, as in the example in the beginning of Introduction. This is the key question we address in this paper.

\hfill {\em How can we effectively and efficiently learn Plackett-Luce and its mixtures from structured partial orders of different kinds?}

Successfully addressing this question faces two challenges. First, to address the effectiveness concern, we need a statistical model that combines various structured partial orders to prove desirable statistical properties, and we are unaware of an existing one. Second, to address the efficiency concern, we need to design new algorithms as either previous algorithms cannot be directly applied, or it is unclear whether the theoretical guarantee such as consistency will be retained.

\subsection{Our Contributions}
Our contributions in addressing the key question are three-fold.

\noindent{\bf Modeling Contributions.} We propose a class of statistical models to model the co-existence of the following three types of structured partial orders mentioned in the Introduction: ranked-top-$l$, $l$-way, and choice-$l$, by leveraging mixtures of Plackett-Luce models. Our models can be easily generalized to include other types of structured partial orders.

\noindent{\bf Theoretical Contributions.} Our main theoretical results characterize the {\em identifiability} of the proposed models. Identifiability is fundamental in parameter estimation, which states that different parameters of the model should give different distributions over data. Clearly, if a model is non-identifiable, then no parameter estimation algorithm can be consistent.  

We prove that when only ranked top-$l_1$ and $l_2$-way ($l_2$ is set to $1$ if there are no $l_2$-way orders) orders are available, the mixture of $k$ Plackett-Luce models is not identifiable if $k\ge (l_1+l_2+1)/2$ (Theorem~\ref{thm:nonid}). We also prove that the mixtures of two Plackett-Luce models is identifiable under the following combinations of structures: ranked top-$3$ (Theorem~\ref{thm:id} (a) extended from \cite{Zhao16:Learning}), ranked top-$2$ plus $2$ way (Theorem~\ref{thm:id} (b)), choice-$2,3,4$ (Theorem~\ref{thm:id} (c)), and 4-way (Theorem~\ref{thm:id} (d)). For the case of mixtures of $k$ Plackett-Luce models over $m$ alternatives, we prove that if there exist $m'\le m$ s.t. the mixture of $k$ Plackett-Luce models over $m'$ alternatives is identifiable, we can learn the parameter using ranked top-$l_1$ and $l_2$-way orders where $l_1+l_2\ge m'$ (Theorem~\ref{thm:idm}). This theorem, combined with Theorem 3 in \cite{Zhao16:Learning}, which provides a condition for mixtures of $k$ Plackett-Luce models to be generically identifiable, can guide the algorithm design for mixtures of arbitrary $k$ Plackett-Luce models.

\noindent{\bf Algorithmic Contributions.} We propose efficient generalized-method-of-moments (GMM) algorithms for parameter estimation of the proposed model based on $2$-PL. Our algorithm runs much faster while providing better statistical efficiency than the EM-algorithm proposed by~\citet{Liu19:Learning} on datasets with large numbers of structured partial orders, see Section~\ref{sec:exp} for more details. Our algorithms are compared with the GMM algorithm by \citet{Zhao16:Learning} under two different settings. When full rankings are available, our algorithms outperform the GMM algorithm by \citet{Zhao16:Learning} in terms of MSE. When only structured partial orders are available, the GMM algorithm by~\citet{Zhao16:Learning} is the best. We believe this difference is caused by the intrinsic information in the data.

% and E-LSR algorithm for learning mixtures of arbitrary $k$ Plackett-Luce models ($k$-PLs).

\subsection{Related Work and Discussions}

%The question is relevant. While the format of data can be enforced by the user interface used to collect data, it is much more convenient and efficient for the user to be able to report the type of partial orders that best represent their preferences. For example, various voting websites allows users to use different UIs to rank a subset of alternatives, choose the top among a set of alternatives, or give a collection of pairwise comparisons. In such cases, it is unclear how rank aggregation should be done and the following question remains open.

%The Plackett-Luce model and its mixtures have been widely studied in statistics~\cite{Hunter04:MM}, psychometrics~\cite{Mollica17:Bayesian}, and algebraic geometry~\cite{Sturmfels12:Commutative}. 

%Mixture models have been widely used for clustering and prediction in many areas, including data mining~\citep{Zhai04:Cross}, bioinformatics~\citep{Mclachlan02:Mixture}, computer vision~\citep{Stauffer99:Adaptive,Kaewtrakulpong01:Improved,zivkovic04:Improved}, etc. 
\noindent{\bf Modeling.}  We are not aware of a previous model targeting rank data that consists of different types of structured partial orders. We believe that modeling the coexistence of different types of structured partial orders is highly important and practical, as it is more convenient, efficient, and accurate for an agent to report her preferences as a structured partial order of her choice. For example, some voting websites allow users to use different UIs to submit structured partial orders~\cite{Brandt2015:Pnyx}.  
%There is a large literature on learning Plackett-Luce model and its mixtures (from linear orders)~\cite{}.\lirong{Zibing please add references.} For rank aggregation from 

There are two major lines of research in rank aggregation from partial orders: learning from structured partial orders and EM algorithms for general partial orders.  Popular structured partial orders investigated in the literature are pairwise comparisons~\cite{Jang16:Top,Jamieson11:Active}, top-$l$~\cite{Mollica17:Bayesian,Huang11:Efficient}, $l$-way~\cite{Marden95:Analyzing,Hunter04:MM,Maystre15:Fast}, and choice-$l$~\cite{Train09:Discrete}. \citet{Khetan16:Data} focused on partial orders with ``separators", which is a broader class of partial orders than top-$k$. But still, \cite{Khetan16:Data} assumes the same structure for everyone. Our model is more general as it allows the coexistence of different types of structured partial orders in the dataset. EM algorithms have been designed for learning mixtures of Mallows' model~\cite{Lu14:Effective} and mixtures of random utility models including the Plackett-Luce model~\cite{Liu19:Learning}, from general partial orders. Our model is less general, but as EM algorithms are often slow and it is unclear whether they are consistent, our model allows for theoretically and practically more efficient algorithms. We believe that our approach provides a principled balance between the flexibility of modeling and the efficiency of algorithms.

%multiple structures of partial orders following mixtures of Plackett-Luce distributions. There are two ways of sampling partial orders: sampling each pairwise comparison independently which results in a general partial order~\cite{Lu14:Effective,Liu19:Learning}, or fix the structure for everyone and project the linear order to the structure~\cite{Khetan16:Data}. Neither of these two ways serves the purpose of this paper.

\noindent{\bf Theoretical results.}  Several previous works provided theoretical guarantees such as identifiability and sample complexity of mixtures of Plackett-Luce models and their extensions to structured partial orders. For linear orders, \citet{Zhao16:Learning} proved that the mixture of $k$ Plackett-Luce models over $m$ alternatives is not identifiable when $k\le 2m-1$ and this bound is tight for $k=2$. We extend their results to the case of structured partial orders of various types. \citet[Theorem~1]{Ammar14:What} proved that when $m=2k$, where $k=2^l$ is a nonnegative integer power of $2$, there exist two different mixtures of $k$ Plackett-Luce models parameters that have the same distribution over $(2l+1)$-way orders. Our Theorem~\ref{thm:nonid} significantly extends this result in the following aspects: (i) our results includes all possible values of $k$ rather than powers of $2$; (ii) we show that the model is not identifiable even under $(2^{l+1}-1)$-way (in contrast to $(2l+1)$-way) orders; (iii) we allow for combinations of ranked top-$l_1$ and $l_2$-way structures. \citet{Oh14:Learning} showed that mixtures of Plackett-Luce models are in general not identifiable given partial orders, but under some conditions on the data, the parameter can be learned using pairwise comparisons. We consider many more structures than pairwise comparisons.
%\citet{Oh14:Learning} proved that under some conditions on data in the context of tensor decomposition, a mixture of Plackett-Luce models can be learned (identified) using only pairwise comparisons. In contrast, we characterize conditions on partial orders given the number of components $k$ beyond the tensor decomposition context. 

Recently, \citet{Chierichetti18:Learning} proved that at least $O(m^2)$ {\em random} marginal probabilities of partial orders are required to identify the parameter of {\em uniform} mixture of two Plackett-Luce models. We show that a carefully chosen set of $O(m)$ marginal probabilities can be sufficient to identify the parameter of {\em nonuniform} mixtures of Plackett-Luce models, which is a significant improvement. Further, our proposed algorithm can be easily modified to handle the case of uniform mixtures. \citet{Zhao18:Learning} characterized the conditions when mixtures of random utility models are generically identifiable. We focus on strict identifiability, which is stronger.

%Various algorithms that learn single Plackett-Luce models from both full rankings and partial orders have been proposed \cite{Cheng10:Label,Negahban17:Rank,Azari13:Generalized,Maystre15:Fast,Khetan16:Data,Khetan16:Computational,Zhao18:Composite, Liu19:Near, Liu19:Learning}. 

\noindent{\bf Algorithms.}  Several learning algorithms for mixtures of Plackett-Luce models have been proposed, including tensor decomposition based algorithm~\cite{Oh14:Learning}, a polynomial system solving algorithm~\cite{Chierichetti18:Learning}, a GMM algorithm~\cite{Zhao16:Learning}, and EM-based algorithms~\cite{Gormley08:Exploring,Tkachenko16:Plackett,Mollica17:Bayesian, Liu19:Learning}. In particular, \citet{Liu19:Learning} proposed an EM-based algorithm to learn from general partial orders. However, it is unclear whether their algorithm is consistent (as for most EM algorithms), and their algorithm is significantly slower than ours. Our algorithms for linear orders are similar to the one proposed by~\citet{Zhao16:Learning}, but we consider different sets of marginal probabilities and our algorithms significantly outperforms the one by \citet{Zhao16:Learning} w.r.t. MSE while taking similar running time.

\section{Preliminaries}
\label{sec:prelim}

Let $\ma=\{a_1, a_2, \ldots, a_m\}$ denote a set of $m$ alternatives and $\ml(\ma)$ denote the set of all linear orders (full rankings) over $\ma$, which are antisymmetric, transitive and total binary relations. A linear order $R\in\ml(\ma)$ is denoted as $a_{i_1}\succ a_{i_2}\succ\ldots\succ a_{i_m}$, where $a_{i_1}$ is the most preferred alternative and $a_{i_m}$ is the least preferred alternative. A partial order $O$ is an antisymmetric and transitive binary relation. In this paper, we consider three types of strict partial orders: ranked-top-$l$ (top-$l$ for short), $l$-way, and choice-$l$, where $l\le m$. A top-$l$ order is denoted by $O^{\topl}=[a_{i_1}\succ\ldots\succ a_{i_l}\succ\text{others}]$; an $l$-way order is denoted by $O^{\lway}=[a_{i_1}\succ\ldots\succ a_{i_l}]$, which means that the agent does not have preferences over unranked alternatives; 
and a choice-$l$ order is denoted by $O^{\text{choice}-l}_{\ma'}=(\ma',a)$, where $\ma'\subseteq \ma$, $|\ma'|=l$, and $a\in \ma'$, which means that the agent chooses $a$ from $\ma'$. 
%In this paper we will use $R$ to denote a linear order and use $O$ to denote a partial order. 
We note that the three types of partial orders are not mutually exclusive. For example, a pairwise comparison is a $2$-way order as well as a choice-$2$ order. Let $\mcp(\ma)$ denote the set of all partial orders of the three structures: ranked top-$l$, $l$-way, and choice-$l$ ($l\le m$) over $\ma$. It is worth noting that $\ml(\ma)\subseteq\mcp(\ma)$.
Let $P=(O_1, O_2, \ldots, O_n)\in \mcp(\ma)^n$ denote the data, also called a {\em preference profile}. Let $O^s_{\ma'}$ denote a partial order over a subset $\ma'$ whose structure is $s$. When $s$ is top-$l$, $\ma'$ is set to be $\ma$. Let $[d]$ denote the set $\{1, 2, \ldots, d\}$.

\begin{dfn} (Plackett-Luce model). 
The parameter space is $\Theta=\{\vec\theta=\{\theta_i|1\le i\le m, 0<\theta_i<1, \sum^m_{i=1}\theta_i=1\}\}$. The sample space is $\ml(\ma)^n$. Given a parameter $\vec\theta\in\Theta$, the probability of any linear order $R=[a_{i_1}\succ a_{i_2}\succ\ldots\succ a_{i_m}]$ is 
$$
\Pr\nolimits_{\text{PL}}(R|\vec\theta)=\prod^{m-1}_{p=1}\frac {\theta_{i_p}} {\sum^m_{q=p}\theta_{i_q}}.$$
\end{dfn}

%The relations between the three types of partial orders are summarized as follows.
%\lirong{the relations are confusing---one cannot talk about event by putting the context in a  statistical model. Also, such relations are only our definition/viewpoint, it doesn not mean that they are general facts. Maybe we can discuss it after introducing PL as is now.}

%\noindent{\bf Relation 1.} For any $l'\le l$, ranked top-$l'$ orders are marginal events of ranked top-$l$ orders;\\
%\noindent{\bf Relation 2.} Given any $\ma'\subset\ma$ and any $\ma''\subset\ma'$, a $|\ma''|$-way order over $\ma''$ can be viewed as a marginal event of $|\ma'|$-way orders over $\ma'$;\\
%\noindent{\bf Relation 3.} Given any $\ma'\subset\ma$ and any $\ma''\subset\ma'$, a choice data over $\ma''$ can be viewed as a marginal event of $|\ma'|$-way rankings over $\ma'$. 

%\subsection{Plackett-Luce Model and Its Mixtures}

Under Plackett-Luce model, a partial order $O$ can be viewed as a marginal event which consists of all linear orders that {\em extend} $O$, that is, for any extension $R$, $a\succ_O b$ implies $a\succ_R b$. The probabilities of the aforementioned three types of partial orders are as follows~\cite{Xia2019:Learning}.
\begin{itemize}
    \item {\bf Top-$l$.} For any top-$l$ order $O^{\topl}=[a_{i_1}\succ\ldots\succ a_{i_l}\succ\text{others}]$, we have
$$\Pr\nolimits_{\text{PL}}(O^\topl|\vec\theta)=\prod^{l}_{p=1}\frac {\theta_{i_p}} {\sum^m_{q=p}\theta_{i_q}}.$$
    \item {\bf $l$-way.} For any $l$-way order $O^\lway_{\ma'}=[a_{i_1}\succ\ldots\succ a_{i_l}]$, where $\ma'=\{a_{i_1},\ldots, a_{i_l}\}$,  we have 
    $$\Pr\nolimits_{\text{PL}}(O^\lway_{\ma'}|\vec\theta)=\prod^{l-1}_{p=1}\frac {\theta_{i_p}} {\sum^l_{q=p}\theta_{i_q}}.$$
    \item {\bf Choice-$l$.} For any choice order $O=(\ma',a_i)$, we have 
    $$\Pr\nolimits_{\text{PL}}(O|\vec\theta)=\frac {\theta_i} {\sum_{a_j\in \ma'}\theta_j}.$$
\end{itemize}

In this paper, we assume that data points are i.i.d.~generated from the model.
% We assume that all partial (or linear) orders in the data are i.i.d.~generated in the Plackett-Luce model. Then given $P$ and $\vec\theta$, we have $\Pr_{\text{PL}}(P|\vec\theta)=\prod^n_{j=1}\Pr_{\text{PL}}(R_j|\vec\theta)$.

%We recall the definition of $k$-PL from \cite{Zhao16:Learning} as follows
\begin{dfn}[Mixtures of $k$ Plackett-Luce models for linear orders ($k$-PL)]
Given $m\ge 2$ and $k\in\mathbb N_+$, the sample space of $k$-PL is $\ml(\ma)^n$. The parameter space is $\Theta=\{\vec\theta=(\vec\alpha, \vec\theta^{(1)}, \ldots, \vec\theta^{(k)})\}$, where $\vec\alpha=(\alpha_1, \ldots, \alpha_k)$ is the {\em mixing coefficients}. For all $r\le k$, $\alpha_r\ge 0$ and $\sum_{r=1}^k\alpha_r=1$. For all $1\le r\le k$, $\vec\theta^{(r)}$ is the parameter of the $r$th Plackett-Luce component. The probability of a linear order $R$ is:
$$
\Pr\nolimits_{k\text{-PL}}(R|\vec\theta)=\sum^k_{r=1}\alpha_r\Pr\nolimits_{\text{PL}}(R|\vec\theta^{(r)}).
$$
\end{dfn}

%\begin{ex}\label{ex:2pl} (A 2-PL over two alternatives) Let the mixing coefficients be $\vec\alpha = (0.75, 0.25)$. The first Plackett-Luce component is parameterized by $\vec\theta^{(1)}=(0.1, 0.9)$, and the second $\vec\theta^{(2)}=(0.7, 0.3)$. Then we have $\Pr_{2-\text{PL}}(a_1\succ a_2) = 0.75\times\frac {0.1} {0.1+0.9}+0.25\times\frac {0.7} {0.7+0.3}=0.25$ and $\Pr_{2-\text{PL}}(a_2\succ a_1) = 1-\Pr(a_1\succ a_2) = 0.75$.
%\end{ex}

%\subsection{Identifiability}

We now recall the definition of identifiability of statistical models.

\begin{dfn}[Identifiability]\label{def:id}
Let $\mathcal{M}=\{\Pr(\cdot|\vec{\theta}): \vec{\theta}\in\Theta\}$ be a statistical model, where $\Theta$ is the parameter space and $\Pr(\cdot|\vec{\theta})$ is the distribution over the sample space associated with $\vec{\theta}\in \Theta$. $\mathcal{M}$ is {\em identifiable} if for all $\vec{\theta}, \vec{\gamma}\in\Theta$, we have 
$$
\Pr(\cdot|\vec{\theta})=\Pr(\cdot|\vec{\gamma})\Longrightarrow
\vec{\theta}=\vec{\gamma}.$$
\end{dfn}

A mixture model is generally not identifiable due to the label switching problem~\citep{Redner84:Mixture}, which means that labeling the components differently leads to the same distribution over data. In this paper, we consider identifiability of mixture models {\em modulo label switching}. That is, in Definition~\ref{def:id}, we further require that $\vec{\theta}$ and $ \vec{\gamma}$ cannot be obtained from each other by label switching. 

%In practice, distribution over the sample space is usually not available. E.g., only partial orders are observed for learning a ranking model. In such cases, we also care about identifiability of a model w.r.t. marginal events, i.e., the data. 

%For any model that is identifiable, it is usually not necessary to use all to learn the ground truth parameter. For example, a Plackett-Luce model over any number of alternatives can be learned using only pairwise comparisons. We can say the Plackett-Luce model is identifiable w.r.t. pairwise comparisons. For convenience, we define {\em identifiability w.r.t. moments}.

%\begin{dfn} (Identifiability w.r.t. data)
%Given a statistical model $\mm=\{\Pr(\cdot|\vec{\theta}): \vec{\theta}\in\Theta\}$ and a set of $q$ marginal events $M=\{\me_1, \ldots, \me_q\}$. $\mm$ is identifiable w.r.t. moments $M$ if for any distinct $\vec\theta, \vec\gamma\in\Theta$, there exists $\zeta\le q$ s.t. $\Pr(\me_\zeta|\vec\theta)\ne\Pr(\me_\zeta|\vec\theta)$.
%\end{dfn}

%Strict identifiability of a ranking model can be viewed as identifiability w.r.t. the $m!$ linear orders. And identifiability w.r.t. a set of marginal events implies strict identifiability.

\section{Mixtures of Plackett-Luce Models for Partial Orders}

We propose the class of mixtures of Plackett-Luce models for the aforementioned structures of partial orders. To this end, each such model should be described by the collection of allowable types of structured partial orders, denoted by $\Phi$. More precisely, $\Phi$ is a set of $u$ structures $\Phi=\{(s_1, \ma_1), \ldots, (s_u, \ma_u)\}$, where for any $t\in[u]$, $(s_t, \ma_t)$ means structure $s_t$ over $\ma_t$. For the case of top-$l$, $\ma_t$ is set to be $\ma$. Since the three structured considered in this paper are not mutually exclusive, {\bf we require that $\Phi$ does not include any pair of overlapping structures simultaneously} for the model to be identifiable. There are two types of pairs of overlapping structures: (1) $(\text{top-}(m-1), \ma)$ and $(m\text{-way}, \ma)$; and (2) for any subset of two alternatives $\ma'$, $(2\text{-way}, \ma')$ and $(\text{choice-}2, \ma')$. Each structure corresponds to a number $\phi^{s_t}_{\ma_t}>0$ and we require $\sum^u_{t=1}\phi^{s_t}_{\ma_t}=1$. A partial order is generated in two stages as illustrated in Figure~\ref{fig:model}: (i) a linear order $R$ is generated by $k$-PL given $\vec\alpha, \vec\theta^{(1)}, \ldots, \vec\theta^{(k)}$; (ii) with probability $\phi^{s_t}_{\ma_t}$, $R$ is projected to the randomly-generated partial order structure $(s_t$, $\ma_t)$ to obtain a partial order $O$.
%$\ptopl{l_1}$, the linear order is projected to a ranked top-$l_1$ order, i.e., keeping the top $l_1$ alternatives in the linear order; with probability $\plway{l_2}{\ma'}$, the linear order is projected to an $l_2$-way order over $\ma'$, i.e., keeping the comparisons only between the alternatives in $\ma'$; and with probability $\pchoice{\ma''}$, the linear order is projected to a choice over $\ma''$, i.e., the top ranked alternative within $\ma''$. 
Formally, the model is defined as follows.

\begin{dfn}[Mixtures of $k$ Plackett-Luce models for partial orders by $\Phi$ (\kplp)]
Given $m\ge 2$, $k\in\mathbb N_+$, and the set of structures $\Phi=\{(s_1, \ma_1), \ldots, (s_u, \ma_u)\}$, the sample space is all structured partial orders defined by $\Phi$. Given $l_1\in [m-1], l
_2, l_3\in [m]$, the parameter space is $\Theta=\{\vec\theta = (\vec\phi, \vec\alpha, \vec\theta^{(1)}, \ldots, \vec\theta^{(k)})\}$. The first part is a vector $\vec\phi=(\phi^{s_1}_{\ma_1}, \ldots, \phi^{s_u}_{\ma_u})$, whose entries are all positive and $\sum^u_{t=1}\phi^{s_t}_{\ma_t}=1$. The second part is $\vec\alpha=(\alpha_1, \ldots, \alpha_k)$ where for all $r\le k$, $\alpha_r>0$ and $\sum_{r=1}^k\alpha_r=1$. The remaining part is $(\vec\theta^{(1)}, \ldots, \vec\theta^{(k)})$, where $\vec\theta^{(r)}$ is the parameter of the $r$th Plackett-Luce component. Then the probability of any partial order $O$, whose structure is defined by $(s, \ma')$, is
$$
\Pr\nolimits_{k\text{-PL-}\Phi}(O|\vec\theta)=\phi^s_{\ma'}\sum^k_{r=1}\alpha_r\Pr\nolimits_{\text{PL}}(O^{s}_{\ma'}|\vec\theta^{(r)}).
$$
% the probabilities of sampling a ranked top-$l_1$ order $R^{\text{top}-l_1}$, an $l_2$-way order $R^{l_2-\text{way}}_{\ma'}$, and a choice-$l$ order $R^{\text{choice}}_{\ma'}$ are
%\begin{align*}
%\Pr\nolimits_{k-\text{PL-partial}}(R^{\text{top}-l_1}|\vec\theta)&=\ptopl{l_1}\sum^k_{r=1}\alpha_r\Pr\nolimits_{\text{PL}}(R^{\text{top}-l_1}|\vec\theta^{(r)})\\
%\Pr\nolimits_{k-\text{PL-partial}}(R^{l_2-\text{way}}|\vec\theta)&=\plway{l_2}{\ma'}\sum^k_{r=1}\alpha_r\Pr\nolimits_{\text{PL}}(R^{l_2-\text{way}}_{\ma'}|\vec\theta^{(r)})\\
%\Pr\nolimits_{k-\text{PL-partial}}(R^{\ma'}|\vec\theta)&=\pchoice{\ma'}\sum^k_{r=1}\alpha_r\Pr\nolimits_{\text{PL}}(R^{
%\text{choice}}_{\ma'}|\vec\theta^{(r)})
%\end{align*}
\end{dfn}

\begin{figure}
    \centering
    \includegraphics[width=0.9\textwidth]{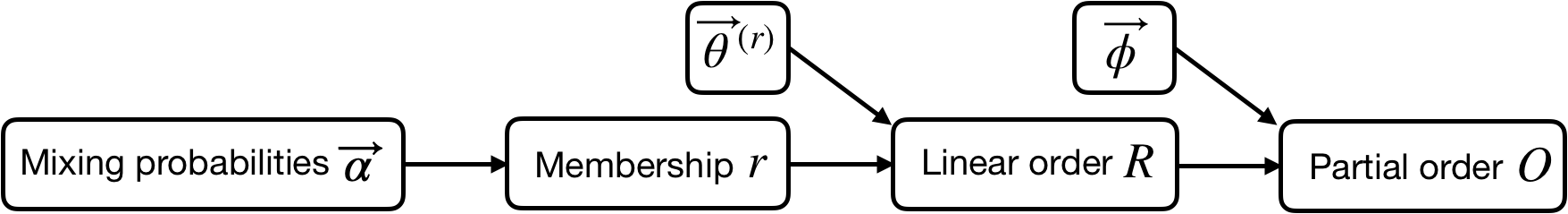}
    \caption{The mixture model for structured partial preferences.}
    \label{fig:model}
\end{figure}

For any partial order $O$ whose structure is $(s, \ma')$, we can also write
\begin{equation}\label{eq:bridge}
\Pr\nolimits_{k\text{-PL-}\Phi}(O|\vec\theta)=\phi^s_{\ma'}\Pr\nolimits_{k\text{-PL}}(O|\vec\theta)
\end{equation}
where $\Pr_{k\text{-PL}}(O|\vec\theta)$ is the marginal probability of $O$ under $k$-PL. This is a class of models because the sample space is different when $\Phi$ is different.

\begin{ex}
Let the set of alternatives be $\{a_1, a_2, a_3, a_4\}$. Consider the 2-PL-$\Phi$ $\mm$ where $\Phi=\{(\text{top-}3, \ma), (\text{top-}2, \ma), (3\text{-way}, \{a_1, a_3, a_4\}), (\text{choice-}3, \{a_1, a_2, a_3\})\}$. $\ptopl{3}=0.2$, $\ptopl{2}=0.1$, $\plway{3}{\{a_1, a_3, a_4\}}=0.3$, $\pchoice{3}{\{a_1, a_2, a_3\}}=0.4$, $\vec\alpha = [\alpha_1, \alpha_2] = [0.2, 0.8]$, $\vec\theta^{(1)}=[0.1, 0.2, 0.3, 0.4]$, $\vec\theta^{(2)}=[0.2, 0.2, 0.3, 0.3]$. Now we compute the probabilities of the following partial orders given the model: $O_1=a_2\succ a_3\succ a_4\succ a_1$ (top-$3$), $O_2=a_4\succ a_3\succ \{a_1, a_2\}$ (top-2), $O_3=a_3\succ a_4\succ a_1$ (3-way), and $O_4=(\{a_1, a_2, a_3\}, a_3)$ (choice-3 over $\{a_1, a_2, a_3\}$). We first compute $\Pr_{\text{PL}}(O_j|\theta^{(r)})$ for all combinations of $j$ and $r$, shown in Table~\ref{tab:pr_single}. 
\begin{table}[!h]
    \centering
    \begin{tabular}{|c|c|c|}
    \hline
         & $r=1$ & $r=2$ \\
    \hline
    $O_1$ & $\frac{0.2}{0.1+0.2+0.3+0.4}\frac{0.3}{0.1+0.3+0.4}\frac{0.4}{0.1+0.4}=0.06$ & $\frac{0.2}{0.2+0.2+0.3+0.3}\frac{0.3}{0.2+0.3+0.3}\frac{0.3}{0.2+0.3}=0.045$\\
    $O_2$ & $\frac{0.4}{0.1+0.2+0.3+0.4}\frac{0.3}{0.1+0.2+0.3}=0.2$ & $\frac{0.3}{0.2+0.2+0.3+0.3}\frac{0.3}{0.2+0.2+0.3}=0.13$\\
    $O_3$ & $\frac {0.3} {0.1+0.3+0.4}\frac {0.4} {0.1+0.4}=0.3$ & $\frac {0.3} {0.2+0.3+0.3}\frac {0.3} {0.2+0.3}=0.225$\\
    $O_4$ & $\frac {0.3} {0.1+0.2+0.3}=0.5$ & $\frac {0.3} {0.2+0.2+0.3}=0.43$ \\
    \hline
    \end{tabular}
    \caption{$\Pr(R_j|\theta^{(r)})$ for all $j=1, 2, 3, 4$ and $r=1, 2$.}
    \label{tab:pr_single}
\end{table}

Let $\Pr_\mm(O_j)$ denote the probability of $O_j$ under model $\mm$, we have
\begin{align*}
\Pr\nolimits_\mm(O_1)&=\ptopl{3}\sum^2_{r=1}\alpha_r\Pr(O_1|\vec\theta^{(r)})=0.2\times(0.2\times 0.06+0.8\times 0.045)=0.0096\\
\Pr\nolimits_\mm(O_2)&=\ptopl{2}\sum^2_{r=1}\alpha_r\Pr(O_2|\vec\theta^{(r)})=0.1\times(0.2\times 0.2+0.8\times 0.13)=0.014\\
\Pr\nolimits_\mm(O_3)&=\plway{2}{\{a_3, a_4\}}\sum^2_{r=1}\alpha_r\Pr(O_3|\vec\theta^{(r)})=0.3\times(0.2\times 0.3+0.8\times 0.225)=0.072\\
\Pr\nolimits_\mm(O_4)&=\pchoice{3}{\{a_1, a_2, a_3\}}\sum^2_{r=1}\alpha_r\Pr(O_4|\vec\theta^{(r)})=0.4\times(0.2\times 0.5+0.8\times 0.43)=0.18
\end{align*}
\end{ex}

\section{(Non-)identifiability of \kplp}

Let $\splway{l}=\{(l\text{-way}, \ma_l)|\ma_l\in\ma, |\ma_l|=l\}$ and $\spchoice{l}=\{(\text{choice-}l, \ma_l)|\ma_l\in\ma, |\ma_l|=l\}$. Given a set of partial orders $E$, we denote a column vector of probabilities of each partial order in $E$ for a Plackett-Luce component with parameter $\vec\theta^{(r)}$ by $\vec f_E(\vec\theta^{(r)})$. Given $\vec\theta^{(1)}, \ldots, \vec\theta^{(2k)}$, we define a $|E|\times 2k$ matrix $\mathbf{F}^k_E$, which is heavily used in the proofs of this paper, by
$
\fek = \begin{bmatrix}\vec f_E(\vec\theta^{(1)}) & \cdots & \vec f_E(\vec\theta^{(2k)})\end{bmatrix}
$.
The following theorem shows that under some conditions on $\Phi$, $k$, and $m$, \kplp\ is not identifiable.

\begin{thm}\label{thm:nonid}
Given a set of $m$ alternatives $\ma$ and any $0\le l_1\le m-1$, $1\le l_2\le m$. Let $\Phi^*=\{(\text{top-}1, \ma), \ldots, (\text{top-}l_1, \ma)\}\cup\splway{1}\cup\ldots\cup\splway{l_2}$. Given any $\Phi\subset\Phi^*$, and for any $k\ge (l_1+l_2+1)/2$, \kplp\ is not identifiable.
\end{thm}

\begin{proof}
It suffices to prove that the theorem holds when $\Phi=\Phi^*$. Given $\Phi^*$, it suffices to prove that the model is not identifiable even if the $\vec\phi$ parameter is unique given the distribution of data.

The proof is constructive. By Lemma 1 of \cite{Zhao16:Learning}, for any $k$ and $m\geq 2k$, we only need to find $\vec\theta^{(1)},\ldots,\vec\theta^{(2k)}$ and $\vec\alpha=[\alpha_1,\ldots,\alpha_{2k}]^T$ such that (1) $\fek\cdot \vec\alpha=0$, where $E$ consists of all ranked top-$l_1$ and $l_2$-way orders, and (2) $\vec \alpha$ has $k$ positive elements and  $k$ negative elements. 

We consider the case where the parameter for first alternative of $r$-th component is $e_r$, where $r=1, 2, \ldots, k$. All other alternatives have the same parameters $b_r=\frac {1-e_r} {m-1}$. 

Table \ref{tab:comp} lists some probabilities (constant factors may be omitted). We can see the probabilities from the two classes have similar structures.

\begin{table}[htp]
\caption{Comparisons between two classes of moments}
\begin{center}
\begin{tabular}{|c|c|}
\hline
$a_1$ top & $e_r$\\
$a_1$ second & $\frac {e_r(1-e_r)} {e_r+m-2}$\\
$a_1$ at position $i$ & $\frac {e_r(1-e_r)^{i-1}} {\prod^{i-1}_{p=1}(pe_r+m-1-p)}$ \\
$a_1$ not in top $l$ & $\frac {(1-e_r)^{l}} {\prod^{l-1}_{p=1}(pe_r+m-1-p)}$\\
\hline
$l$-way $a_1$ top & $\frac {(m-1)e_r} {(m-l)e_r+(l-1)} $\\
$l$-way $a_1$ second & $\frac {(m-1)e_r(1-e_r)} {((m-l)e_r+(l-1))((m-l+1)e_r+(l-2))} $\\
$l$-way $a_1$ at position $i$ & $\frac {(m-1)e_r(1-e_r)^{i-1}} {\prod^{i-1}_{p=0}((m-l+p)e_r+(l-1-p))} $\\
$l$-way $a_1$ at position $l$ & $\frac {(1-e_r)^{l-1}} {\prod^{l-2}_{p=0}((m-l+p)e_r+(l-1-p))} $\\
\hline
\end{tabular}
\end{center}
\label{tab:comp}
\end{table}%

It is not hard to check that the probability for $a_1$ to be ranked at the $i$-th position in the $r$-th component is
\begin{equation}\label{equmain:probl1}
\frac{(m-1)!}{(m-i)!}\frac {e_r(b_r)^{i-1}}
{\prod^{i-1}_{p=0}(1-pb_r)}
\end{equation}
where $1\leq i\leq l_1$. The probability for $a_1$ to be ranked out of top $l_1$ position is $\frac{(m-1)!}{(m-l_1)!}\frac {(b_r)^{l}}
{\prod^{l_1-1}_{p=0}(1-pb_r)}$.

And the probability for $a_1$ to be ranked at the $i$-th position in the $r$-th component for $l_2$-way rankings is
\begin{equation}\label{equmain:probl2}
\frac{(l_2-1)!}{(l_2-i)!}\frac {e_r(b_r)^{i-1}}
{\prod^{l_2-1}_{p=l_2-i}(e_r+pb_r)}
\end{equation}
where $1\leq i\leq l_2$.

Then $\fek$ can be reduced to a $(l_1+l_2+1)\times (2k)$ matrix. We now define a new $(2k-1)\times (2k)$ matrix $\hek$ obtained from $\fek$ by performing the following linear operations on row vectors. (i) Make the first row of $\hek$ to be $\vec 1$; (ii) for any $2\leq i\leq l_1+1$, the $i$-th row of $\hek$ is the probability for $a_1$ to be ranked at the $(i-1)$-th position according to \eqref{equmain:probl1}; (iii) for any $l_1+2\leq i\leq l_1+l_2$, the $i$-th row of $\hek$ is the probability for $a_1$ to be ranked at the $(i-l_1-1)$-th position in an $l_2$-way order according to \eqref{equmain:probl2} ; (iv) the $(l_1+l_2+1)$th row is the probability that $a_1$ is not ranked within top $l_1$; (v) remove all constant factors.

More precisely, for any $e_r$ we define the following function.
\begin{equation*}
\vec{f^*_E}(e_r)=
\begin{pmatrix}
1\\
e_r\\
\frac {e_r(1-e_r)} {e_r+m-2}\\
\vdots\\
\frac {e_r(1-e_r)^{l_1-1}} {\prod^{l_1-1}_{p=1}(pe_r+m-1-p)}\\
\frac {(1-e_r)^{l_1}} {\prod^{l_1-1}_{p=1}(pe_r+m-1-p)}\\
\frac {e_r} {(m-l_2)e_r+(l_2-1)}\\
\vdots\\
\frac {e_r(1-e_r)^{l_2-2}} {\prod^{l_2-2}_{p=0}((m-l_2+p)e_r+(l_2-1-p))}\\
\frac {(1-e_r)^{l}} {\prod^{l-1}_{p=1}(pe_r+m-1-p)}
\end{pmatrix}
\end{equation*}

Then we define $
\hek=[\vec{f^*_E} (e_1), \vec{f^*_E}(e_2), \cdots, \vec{f^{*}_E}(e_{2k})]
$.

For any $r\le 2k$, let

\begin{equation}\label{equ:mainbeta}
\beta^\ast_r=\frac {\prod^{l_1-1}_{p=1} (p e_r+m-1-p)\prod^{l_2-2}_{p=0}((m-l_2+p)e_r+l_2-1-p)} {\prod_{q\neq r} (e_r-e_q)}
\end{equation}
Note that the numerator of $\beta^*_r$ is always positive. W.l.o.g. let $e_1<e_2<\cdots<e_{2k}$, then half of the denominators are positive and the other half are negative. 
Note that the degree of the numerator of $\beta^\ast_r$ is $l_1+l_2-2$. By Lemma 6 of \cite{Zhao16:Learning}, we have $\hek\vec{\beta^\ast}=0$. 
\end{proof}

Considering that any $l$-way order implies a choice-$l$ order, we have the following corollary.

\begin{coro}
Given a set of $m$ alternatives $\ma$ and any $0\le l_1\le m-1$, $1\le l_3\le m$. Let $\Phi^*=\{(\text{top-}1, \ma), \ldots, (\text{top-}l_1, \ma)\}\cup\spchoice{1}\cup\ldots\cup\spchoice{l_3}$. Given any $\Phi\subset\Phi^*$, and for any $k\ge (l_1+l_3+1)/2$, \kplp\ is not identifiable.
\end{coro}

Given any $k$, these results show what structures of data we cannot use if we want to interpret the learned parameter. Next, we will characterize conditions for 2-PL-$\Phi$'s to be identifiable.

\begin{thm}\label{thm:id} Let $\Phi^*$ be one of the four combinations of structures below. For any $\Phi\supset\Phi^*$, 2-PL-$\Phi$ over $m\ge 4$ alternatives is identifiable.\\
(a) $\Phi^*=\{(\text{top-}3, \ma)\}$,
(b) $\Phi^*=\{(\text{top-}2, \ma)\}\cup\splway{2}$,
(c) $\Phi^*=\cup^4_{t=2}\spchoice{t}$, or
(d) $\Phi^*=\splway{4}$.
\end{thm}

\begin{proof}
The proof has two steps. The first step is the same across (a), (b), (c), and (d). We show that for any \kplp\ with any parameter $\vec\theta=(\vec\phi, \vec\alpha, \vec\theta^{(1)}, \vec\theta^{(2)})$, there does not exist $\vec\phi'\ne\vec\phi$ s.t. for any $\vec\theta'=(\vec\phi', \vec\alpha', \vec\theta'^{(1)}, \vec\theta'^{(2)})$ the distribution over the sample space is exactly the same. For the purpose of contradiction suppose such $\vec\phi'$ exists. Since $\vec\phi'\ne\vec\phi$, there exist a structure $(s, \ma_s)$ s.t. $\phi^s_{\ma_s}\ne\phi'^s_{\ma_s}$. Now we consider the total probability of all possible partial orders of this structure, denoted by $O_1, O_2, \ldots, O_w$. Then we have
$$\sum^w_{j=1}\Pr\nolimits_{k\text{-PL-}\Phi}(O_j|\vec\theta)=\phi^s_{\ma_s}\ne\phi'^s_{\ma_s}=\sum^w_{j=1}\Pr\nolimits_{k\text{-PL-}\Phi}(O_j|\vec\theta'),$$ which is a contradiction.

In the second step, we show that for any \kplp with any parameter $\vec\theta=(\vec\phi, \vec\alpha, \vec\theta^{(1)}, \vec\theta^{(2)})$, there does not exist $\vec\alpha', \vec\theta'^{(1)}, \vec\theta'^{(2)}$ s.t. for any $\vec\theta'=(\vec\phi, \vec\alpha', \vec\theta'^{(1)}, \vec\theta'^{(2)})$. We will prove for each of the cases (a), (b), (c), and (d).

\noindent{\bf (a)} This step for (a) is exactly the same as the proof for~\cite[Theorem 2]{Zhao16:Learning}.

\noindent{\bf (b)} We focus on $m=4$. The case for $m>4$ is very similar. Let $E$ consist of all ranked top-$2$ and $2$-way orders ($\frac 3 2 m(m-1)$ marginal probabilities). We will show that for all non-degenerate $\vec\theta^{(1)}, \vec\theta^{(2)}, \vec\theta^{(3)}, \vec\theta^{(4)}$, rank$(\mathbf F^2_E)=4$. Then this part is proved by applying~\cite[Lemma 1]{Zhao16:Learning}.

For simplicity we use $[e_r, b_r, c_r, d_r]^\top$ to denote the parameter of $r$th Plackett-Luce model for $a_1, a_2, a_3, a_4$ respectively, i.e.,
\begin{equation*}
\begin{bmatrix} \vec\theta^{(1)} & \vec\theta^{(2)} & \vec\theta^{(3)} & \vec\theta^{(4)}\end{bmatrix}=\begin{bmatrix}
e_1 & e_2 & e_3 & e_4\\
b_1 & b_2 & b_3 & b_4\\
c_1 & c_2 & c_3 & c_4\\
d_1 & d_2 & d_3 & d_4
\end{bmatrix}
\end{equation*}
We define $\vec 1= [1, 1, 1,1]$ and the following row vectors.
\begin{align*}
\vec{1}&=[1, 1, 1, 1]\\
\vec\omega^{(1)}&=[e_1, e_2, e_3, e_4]\\
\vec\omega^{(2)}&=[b_1, b_2, b_3, d_3]\\
\vec\omega^{(3)}&=[c_1, c_2, c_3, c_4]\\
\vec\omega^{(4)}&=[d_1, d_2, d_3, d_4]
\end{align*}

We have $\sum_{i=1}^4\vec\omega^{(i)}=\vec 1$. Therefore, if there exist three $\vec\omega$'s such that $\{\vec\omega^{(1)},\vec\omega^{(2)},\vec\omega^{(3)}\}$ and $\vec 1$ are linearly independent, then $\rank(\fek)=4$. The proof is done. Because $\vec\theta^{(1)}, \vec\theta^{(2)},\vec\theta^{(3)},\vec\theta^{(4)}$  is non-degenerate, at least one of $\{\vec\omega^{(1)},\vec\omega^{(2)},\vec\omega^{(3)},\vec\omega^{(4)}\}$ is linearly independent of $\vec 1$. W.l.o.g.~suppose $\vec\omega^{(1)}$ is linearly independent of $\vec{1}$. This means that not all of $e_1, e_2, e_3, e_4$ are equal. Following \cite{Zhao16:Learning}, we prove the theorem in the following two cases.

\noindent\textbf{Case 1.} $\vec\omega^{(2)}$, $\vec\omega^{(3)}$, and $\vec\omega^{(4)}$ are all linear combinations of $\vec{1}$ and $\vec\omega^{(1)}$. \\
\noindent{\bf Case 2.} There exists a $\vec\omega^{(i)}$ (where $i\in \{2,3,4\}$) that is linearly independent of $\vec{1}$ and $\vec\omega^{(1)}$.

Case 2 was proved by \citet{Zhao16:Learning} using only ranked top-$2$ orders, as well as most of Case 1. The only remaining case is as follows. For all $r=1,2,3,4$,
\begin{equation}\label{claim13set}
\vec\theta^{(r)}=\begin{bmatrix}e_r\\
b_r\\
c_r\\
d_r
\end{bmatrix}=\begin{bmatrix} e_r\\
p_2 e_r-p_2\\
p_3 e_r-p_3\\
-(1+p_2+p_3)e_r+(1+p_2+p_3)
\end{bmatrix}
\end{equation}

We first show a claim, which is useful to the proof.
\begin{claim}\label{condp}
Under the settings of \eqref{claim13set}, $-1<p_2, p_3<0$ and there exists $p$ in $\{p_2, p_3\}$ s.t. $p\neq -\frac 1 2$.
\end{claim}
\begin{proof}
If $p_2=p_3=-\frac 1 2$, then $d_r=0$, which is a contradiction. Since $e_r<1$ and $b_r, c_r>0$, we have $p_2, p_3<0$. If $p_2\leq -1$ (or $p_3\leq -1$), then $e_r+b_r\geq 1$ (or $e_r+c_r\geq 1$), which means parameters corresponds to all other alternatives are zero or negative. This is a contradiction.
\end{proof}

So if $p_2=-\frac 1 2$, we switch the role of $a_2$ and $a_3$. Then we have $p_2\neq -\frac 1 2$.

In this case, we construct $\hat{\mathbf F}$ in the following way.
\begin{table}[htp]
\begin{center}
\begin{tabular}{|c|c|}
\hline
$\mathbf{\hat F}$ & Moments\\
\hline
\multirow{4}{*}{
$\begin{bmatrix}
1 & 1 & 1 & 1\\
e_1 & e_2 & e_3 & e_4\\
\frac {e_1b_1} {1-b_1} & \frac {e_2b_2} {1-b_2} & \frac {e_3b_3} {1-b_3} & \frac {e_4b_4} {1-b_4}\\
\frac {e_1} {e_1+b_1} & \frac {e_2} {e_2+b_2} & \frac {e_3} {e_3+b_3} & \frac {e_4} {e_4+b_4}
\end{bmatrix}$}
& $\vec 1$\\
& $a_1\succ\text{others}$\\
& $a_2\succ a_1\succ\text{others}$\\
& $a_1\succ a_2$\\
& \\
\hline
\end{tabular}
\end{center}
\label{tabfc}
\end{table}%

Let $\vec\omega^{(1)}=[e_1, e_2, e_3, e_4]$.

Define $\vec\theta^{(b)}$
\begin{align*}
\vec\theta^{(b)}=[\frac 1 {1-b_1}, \frac 1 {1-b_2}, \frac 1 {1-b_3}, \frac 1 {1-b_4}]=[\frac 1 {1-p_2 e_1+p_2}, \frac 1 {1-p_2 e_2+p_2},\frac 1 {1-p_2 e_3+p_2}, \frac 1 {1-p_2 e_4+p_2}]
\end{align*}
And define
\begin{align*}
\vec\theta^{(be)}=[\frac 1 {(p_2+1)e_1-p_2}, \frac 1 {(p_2+1)e_2-p_2},\frac 1 {(p_2+1)e_3-p_2}, \frac 1 {(p_2+1)e_4-p_2}]
\end{align*}
Further define
$
\mathbf{F}^\ast=
\begin{bmatrix}
\vec 1\\
\vec\omega^{(1)}\\
\vec\theta^{(b)}\\
\vec\theta^{(be)}
\end{bmatrix}
$.
We will show $\hat{\mathbf F}=T^*\times\mathbf{F}^\ast$ where $T^*$ has full rank.

The last two rows of $\hat{\mathbf F}$ are
\begin{align*}
\frac {e_rb_r} {1-b_r}&=-e_r-\frac 1 {p_2}+\frac {1+p_2} {p_2(1-p_2 e_r+p_2)}\\
\frac {e_r} {e_r+b_r}&=\frac {e_r} {(p_2+1) e_r-p_2}=\frac 1 {p_2+1}+\frac {p_2} {(p_2+1)((p_2+1)e_r-p_2)}
\end{align*}
So
\begin{equation*}%
\mathbf{\hat F}=
\begin{bmatrix}
\vec 1\\
\vec\omega^{(1)}\\
-\frac 1 {p_2}\vec 1-\vec\omega^{(1)}+\frac {1+p_2} {p_2}\vec\theta^{(b)}\\
\frac {1} {p_2+1}\vec 1+\frac {p_2} {p_2+1}\vec\theta^{(be)}
\end{bmatrix}
\end{equation*}

Then we have $\mathbf{\hat F}=T^*\times\mathbf{F}^\ast$ where
\begin{equation*}
T^*=\begin{bmatrix}
1 & 0 & 0 & 0 \\
0 & 1 & 0 & 0 \\
-\frac 1 {p_2} & -1 & \frac {1+p_2} {p_2} & 0 \\
\frac {1} {p_2+1} & 0 & 0 & \frac {p_2} {p_2+1}\\
\end{bmatrix}
\end{equation*}
From Claim \ref{condp}, we have $-1<p_2<0$. So $\frac {1+p_2} {p_2}, \frac {p_2} {p_2+1}\neq 0$. So $T$ has full rank. Then $\operatorname{rank}(\mathbf{F}^\ast)=\operatorname{rank}(\mathbf{\hat F})$. 

If rank$(\mathbf{F}^{2}_4)\leq 3$, then there is at least one column in $\mathbf{F}^2_4$ dependent of other columns. As all rows in $\hat{\mathbf F}$ are linear combinations of rows in $\mathbf{F}^2_4$, rank$(\mathbf{\hat F})\leq 3$. Since $\operatorname{rank}(\mathbf{F}^\ast)=\operatorname{rank}(\mathbf{\hat F})$, we have rank$(\mathbf{F}^\ast)\leq 3$. Therefore, there exists a nonzero row vector $\vec{t}=[t_1, t_2, t_3, t_4]$, s.t. 
\begin{equation*}
\vec{t}\mathbf{F}^\ast=0
\end{equation*}
Namely, for all $r\leq 4$,
\begin{equation*}%
t_1+t_2 e_r+\frac {t_3} {1-p_2e_r+p_2}+\frac {t_4} {(p_2+1)e_r-p_2}=0
\end{equation*}
Let
\begin{align*}
f(x)&=t_1+t_2x+\frac {t_3} {1-p_2x+p_2}+\frac {t_4} {(p_2+1)e_r-p_2}\\
g(x)&=(1-p_2x+p_2)((p_2+1)e_r-p_2)(t_1+t_2x)+t_3((p_2+1)e_r-p_2)+t_4(1-p_2x+p_2)
\end{align*}
If any of the coefficients of $g(x)$ is nonzero, then $g(x)$ is a polynomial of degree at most 3. There will be a maximum of $3$ different roots. As the equation holds for all $e_r$ where $r=1,2,3,4$. There exists $s\neq t$ s.t. $e_s=e_t$. Otherwise $g(x)=f(x)=0$ for all $x$. We have
\begin{align*}
g(\frac {1+p_2} {p_2})&=\frac {t_3(1+2p_2)} {p_2}=0\\
g(\frac {p_2} {p_2+1})&=\frac {t_4p_2} {p_2+1}=0
\end{align*}
From Claim \ref{condp} we know $p_2<0$ and $p_2\neq -\frac 1 2$. So $t_3=t_4=0$. Substitute it into $f(x)$ we have $f(x)=t_1+t_2x=0$ for all $x$. So $t_1=t_2=0$. This contradicts the nonzero requirement of $\vec t$. Therefore there exists $s\neq t$ s.t. $e_s=e_t$. We have $\vec\theta^{(s)}=\vec\theta^{(t)}$, which is a contradiction.

\noindent{\bf (c)} We prove this theorem by showing that the marginal probabilities of partial orders from Theorem~\ref{thm:id} (b) can be derived from the marginal probabilities in this theorem. 

It is not hard to check the following equation holds considering any subset of four alternatives $\{a_{i_1}, a_{i_2}, a_{i_3}, a_{i_4}\}$.
$$
\Pr(a_{i_1}\succ a_{i_2}\succ\{a_{i_3}, a_{i_4}\}
=\Pr(a_{i_2}\succ\{a_{i_3}, a_{i_4}\})-\Pr(a_{i_2}\succ\{a_{i_1}, a_{i_3}, a_{i_4}\})
$$
The intuition is that the probability of $a_{i_2}$ being selected given $\{a_{i_2}, a_{i_3}, a_{i_4}\}$ can be decomposed into two parts: the probability of $a_{i_2}$ being selected given $\{a_{i_1}, a_{i_2}, a_{i_3}, a_{i_4}\}$ and the probability of $a_{i_2}$ being ranked at the second position and $a_{i_1}$ being ranked at the first position. This equation means we can obtain the probabilities ranked top-$2$ orders over a subset of four alternatives using choice data over the subset of alternatives. Then if we treat this four alternatives as a $2$-PL, the parameter is identifiable.

In the case of more than four alternatives, we first group the alternatives into subsets of four and one arbitrary alternative is included in all groups. For example, when $m=6$, we can make it two subsets: $\{a_1, a_2, a_3, a_4\}, \{a_1, a_5, a_6, a_2\}$. It is okay to have more than one overlapping alternatives, but in practice we hope to have as few groups as possible for considerations of computational efficiency. The parameter of each subset of alternatives can be uniquely learned up to a scaling factor. For any $r$, it is not hard to scale $\theta^{(r)}_i$ for all $i$ s.t. $\theta^{(r)}_1$ is the same for all groups and $\sum^m_{i=1}\theta^{(r)}_i=1$.

\noindent{\bf (d)} This is proved by applying the fact that any 4-way order implies a set of choice-2,3,4 orders to (c).
\end{proof}

Identifiability for $k\ge 3$ is still an open question and \citet{Zhao16:Learning} proved that when $k\le\lfloor\frac {m-2} 2\rfloor!$, generic identifiability holds for $k$-PL, which means the Lebesgue measure of non-identifiable parameter is zero. We have the following theorem that can guide algorithm design for \kplp. Full proof of Theorem~\ref{thm:idm} can be found in the appendix.

\begin{thm}\label{thm:idm}
Let $l_1\in[m-1]$, $l_2\in[m]$, and $\Phi^*=\{(\text{top-}l_1, \ma), (l_2\text{-way}, \ma')|\ma'\in\ma, |\ma'|=l_2\}$. Given any $\Phi\supset\Phi^*$, if $k$-PL over $m'$ alternatives is (generically) identifiable, $k$-PL-$\Phi$ over $m\ge m'$ alternatives is (generically) identifiable when $l_1+l_2\ge m'$.
\end{thm}

\begin{proof} As was proved in Theorem~\ref{thm:id}, the $\vec\phi$ parameter is identifiable. Now we prove that $\vec\alpha, \vec\theta^{(1)}, \ldots, \vec\theta^{(k)}$ is (generically) identifiable.

The set of partial orders where $l_1+l_2=m'$ is a subset of partial orders where $l_1+l_2\geq m'$, so we only need to prove the cases where $l_1+l_2=m$. We prove this theorem by induction.

Recall that $1\leq l_2\leq m$. If $l_2=1$, then $l_1=m-1$, meaning this set of partial orders includes all linear rankings. The parameter is identifiable. This case serves as the base case.

Assume this theorem holds for a certain $l_1=u$ and $l_2=v$ where $u+v=m$, then consider the set of partial orders where $l_1=u-1, l_2=v+1$. This case adds ($v+1$)-way orders but leaves out ranked top-$u$ orders. We can recover ranked top-$u$ rankings using ranked top-($u-1$) and ($v+1$)-way orders in the following way.

Suppose we need to recover a ranked top-$u$ order $a_1\succ a_2\succ\cdots\succ a_u\succ\text{others}$. The remaining alternatives are $a_{u+1}, a_{u+1}, \cdots, a_{u+v}$. Let $U=\{a_1, a_2, \cdots, a_{u-1}\}$ and $V=\{a_{u+1}, a_{u+2}, \cdots, a_{u+v}\}$. Then we have $
\Pr(a_1\succ a_2\succ\cdots\succ a_u\succ\text{others})+\sum^{u-1}_{i=1}\Pr(a_u\text{ at $i$-th position, first $i-1$ alternatives $\in U$})
=\Pr(a_u\text{ to be ranked top in }\{a_u\}\cap V)$. Then the parameter can be learned in this case.
\end{proof}

%Theorem~\ref{thm:idm} appears to be weak. However, we can group the alternatives into subsets of $m'$ alternatives with one arbitrary overlapping alternatives as we did in the proof of Theorem~\ref{thm:id} (c). Then as $m$ increases, the number of marginal probabilities needed will only increase linearly with $m$. We note that the marginal probabilities needed for Theorems~\ref{thm:id} (a) and \ref{thm:id} (b) is $O(m^2)$. As we will see in our experiments, the runtimes of our proposed algorithms are closely related to the number of marginal probabilities. So an algorithm guided by Theorem~\ref{thm:idm} can be significantly faster than using probabilities of full rankings ($m!$ in total) when $m$ is large.

%Further, we consider each subgroup, i.e., learning from rankings over $m'$ alternatives. Then the number of all full rankings is $m'!$. If we use ranked top-$\frac {m'} 2$ and $\frac {m'} 2$-way orders (suppose $m'$ is even), there will be $(\frac {m'} 2)!$ marginal probabilities, which is significantly smaller than $m'!$.

\section{Consistent Algorithms for Learning 2-PL-$\Phi$}

We propose a two-stage estimation algorithm. In the first stage, we make one pass of the dataset to determine $\Phi$ and estimate $\vec\phi$. In the second stage, we estimate the parameter $\vec\theta$. We note that these two stages only require one pass of the data.

In the first stage we check the existence of each structure in the dataset and estimate $\ptopl{l}$, $\plway{l}{\ma'}$, and $\pchoice{l}{\ma''}$ for any $l$, $\ma'$ and $\ma''$ by dividing the occurrences of each structure by the size of the dataset. Formally, for any structure $(s, \ma_s)$,
\begin{equation}\label{eq:firststage}
\phi^s_{\ma_s}=\frac {\# \text{ of orders with structure }(s, \ma_s)} {n}
\end{equation}
%\begin{align}\label{eq:firststage}
%\ptopl{l}&=\frac {\# \text{ of ranked top-}l\text{ orders}} {n}\notag\\
%\plway{l}{\ma'}&=\frac {\# \text{ of }l\text{-way}\text{ orders over }\ma'} {n}\notag\\
%\pchoice{\ma'}&=\frac {\# \text{ of choice-}l\text{ orders over }\ma'} {n}
%\end{align}

In the second stage, we estimate $\vec\theta$ using the generalized-method-of-moments (GMM) algorithm. In a GMM algorithm, a set of $q$ marginal events (partial orders in the case of rank data), denoted by $E=\{\me_1, \ldots, \me_q\}$, are selected. Then $q$ {\em moment conditions} $\vec g(O, \vec\theta)\in\mathbb R^q$, which are functions of a data point $O$ and the parameter $\vec\theta$, are designed. The expectation of any moment condition is zero at the ground truth parameter $\vec\theta^*$, i.e., $E[g(O, \vec\theta^*)]=\vec 0$. For a dataset $P$ with $n$ rankings, we let $\vec g(P, \vec\theta)=\frac 1 n\sum_{O\in P}g(O, \vec\theta)$. Then the estimate is $\hat\theta=\arg\min||g(P, \vec\theta)||^2_2$.

Now we define moment conditions $\vec g(O, \vec\theta)$. For any $t
\le q$, the $t$-th moment condition $g_t(O, \vec\theta)$ corresponds to the event $\me_t$. Let $(s_t, \ma_t)$ denote the structure of $\me_t$. If $O=\me_t$, we define $g_t(O, \vec\theta)=\frac 1 {\phi^{s_t}_{\ma_t}}(\Pr_{k\text{-PL-}\Phi}(\me_t|\vec\theta)-1)$; otherwise $g_t(O, \vec\theta)=\frac 1 {\phi^{s_t}_{\ma_t}}\Pr_{k\text{-PL-}\Phi}(\me_t|\vec\theta)$. Under this definition, we have
\begin{equation}\label{eq:secondstage}
\vec\theta'=\arg\min\sum^q_{t=1}(\frac {\Pr_{k\text{-PL-}\Phi}(\me_t|\vec\theta)} {\phi^{s_t}_{\ma_t}}-\frac {\#\text{ of }\me_t} {n\phi^{s_t}_{\ma_t}})^2
\end{equation}
We consider two ways of selecting $E$ for 2-PL-$\Phi$ guided by our Theorem~\ref{thm:id} (b) and (c) respectively. 

%\noindent{\bf Ranked top-3 only} ($\ptopl{3}=1$). For this case, our algorithm reduces to the GMM algorithm proposed by \cite{Zhao16:Learning}. The selected partial orders are: ranked top-1 for each alternative ($m$ moment conditions), ranked top-2 for all pairs of alternatives ($m(m-1)$ moment conditions), and $m$ ranked top-3 orders, including {\zz $a_1\succ a_2\succ a_3\succ\text{others}, a_2\succ a_3\succ a_4\succ\text{others}, \ldots, a_m\succ a_1\succ a_2\succ\text{others}$}.

\noindent{\bf Ranked top-2 and 2-way} ($\Phi=\{(\text{top-}2, \ma), (2\text{-way}, \ma')|\ma'\in\ma, |\ma'|=2\}$). The selected partial orders are: ranked top-2 for each pair ($m(m-1)-1$ moment conditions) and all combinations of 2-way orders ($m(m-1)/2$ moment conditions). We remove one of the ranked top-2 orders because this corresponding moment condition is linearly dependent of the other ranked top-2 moment conditions. For the same reason, we only choose one for each 2-way comparison, resulting in $m(m-1)/2$ moment conditions. For example. in the case of $\ma=\{a_1, a_2, a_3, a_4\}$, we can choose $E=\{a_1\succ a_2\succ\text{others},a_1\succ a_3\succ\text{others},a_1\succ a_4\succ\text{others},a_2\succ a_1\succ\text{others},a_2\succ a_3\succ\text{others}, a_2\succ a_4\succ\text{others},a_3\succ a_1\succ\text{others},a_3\succ a_2\succ\text{others},a_3\succ a_4\succ\text{others},a_4\succ a_1\succ\text{others},a_4\succ a_2\succ\text{others}, a_1\succ a_2, a_1\succ a_3, a_1\succ a_4, a_2\succ a_3, a_2\succ a_4, a_3\succ a_4\}$.

\noindent{\bf Choice-4.} We first group $\ma$ into subsets of four alternatives so that $a_1$ is included in all subsets. And a small number of groups is desirable for computational considerations. One possible way is $G_1=\{a_1, a_2, a_3, a_4\}$, $G_2=\{a_1, a_5, a_6, a_7\}$, etc. The last group can be $\{a_1, a_{m-2}, a_{m-1}, a_m\}$. More than one overlapping alternatives across groups is fine. In this way we have $\lceil\frac {m-1} 3\rceil$ groups. We will define $\Phi_G$ and $E_G$ for any group $G=\{a_{i_1}, a_{i_2}, a_{i_3}, a_{i_4}\}$. Then $\Phi=\cup^{\lceil\frac {m-1} 3\rceil}_{t=1}\Phi_{G_t}$ and $E=\cup^{\lceil\frac {m-1} 3\rceil}_{t=1}E_{G_t}$. For any $G=\{a_{i_1}, a_{i_2}, a_{i_3}, a_{i_4}\}$, $\Phi_G=\{(\text{choice-}4, G), (\text{choice-}3, G'), (\text{choice-}2, G'')|G', G''\in G, |G'|=3, |G''|=2\}$. $E$ includes all $17$ choice-2,3,4 orders. $E=\{(G, a_{i_1}), (G, a_{i_2}), (G, a_{i_3}), (\{a_{i_1}, a_{i_2}, a_{i_3}\}, a_{i_1}), (\{a_{i_1}, a_{i_2}, a_{i_3}\}, a_{i_2}), \\(\{a_{i_1}, a_{i_2}, a_{i_4}\}, a_{i_1}), (\{a_{i_1}, a_{i_2}, a_{i_4}\}, a_{i_2}), (\{a_{i_1}, a_{i_3}, a_{i_4}\}, a_{i_1}), (\{a_{i_1}, a_{i_3}, a_{i_4}\}, a_{i_3}), (\{a_{i_2}, a_{i_3}, a_{i_4}\}, \\a_{i_2}), (\{a_{i_2}, a_{i_3}, a_{i_4}\}, a_{i_3}), (\{a_{i_1}, a_{i_2}\}, a_{i_1}), (\{a_{i_1}, a_{i_3}\}, a_{i_1}), (\{a_{i_1}, a_{i_4}\}, a_{i_1}), (\{a_{i_2}, a_{i_3}\}, a_{i_2}),\\ (\{a_{i_2}, a_{i_4}\}, a_{i_2}), (\{a_{i_3}, a_{i_4}\}, a_{i_3})\}$.

Formally our algorithms are collectively represented as Algorithm~\ref{alg:partial}. We note that only one pass of data is required for estimating $\vec\phi$ and computing the frequencies of each partial order. The following theorem shows that Algorithm~\ref{alg:partial} is consistent when $E$ is chosen for ``ranked top-2 and 2-way" and ``choice-4".

\begin{algorithm}
{\bf Input}: Preference profile $P$ with $n$ partial orders. A set of preselected partial orders $E$.

{\bf Output}: Estimated parameter $\vec\theta'$.

Estimate $\vec\phi$ using \eqref{eq:firststage}.\\
For each $\me\in E$, compute the frequency of $\me$.\\
Compute the output using \eqref{eq:secondstage}.
\caption{Algorithms for $2$-PL-$\Phi$.}
\label{alg:partial}
\end{algorithm}

\begin{thm}\label{thm:consist}
Given $m\ge 4$. If there exists $\epsilon>0$ s.t. for all $r=1, 2$ and $i=1, \ldots, m$, $\theta^{(r)}_i\in[\epsilon, 1]$, and
$E$ is selected following either of ``ranked top-2 and 2-way" and ``choice-4", 
then Algorithm~\ref{alg:partial} is consistent.
\end{thm}

\begin{proof} We first prove that the estimate of $\vec\phi$ is consistent. Let $X_t$ denote a random variable, where $X_t=1$ if a structure $(s_t, \ma_t)$ is observed and $0$ otherwise. The dataset of $n$ partial orders is considered as $n$ trials. Let the $j$-th observation of $X_t$ be $x_j$. Then we have $E[\frac {\sum^n_{j=1}x_j} n]=\phi^{s_t}_{\ma_t}$, which means as $n\rightarrow\infty$, $\frac {\sum^n_{j=1}x_j} n$ converges to $\phi^{s_t}_{\ma_t}$ with probability approaching one.

Now we prove that the estimation of $\vec\alpha, \vec\theta^{(1)}, \vec\theta^{(2)}$ is also consistent. 

We write the moment conditions $\vec g(P, \vec\theta)$ as $\vec g_n(\vec\theta)$ and define 
$$\vec g_0(\vec\theta)=E[\vec g_n(\vec\theta)].$$
Let $\vec\theta^*$ denote the ground truth parameter. By definition, we have 
$$\vec g_0(\vec\theta^*)=E[\frac {\Pr_{k\text{-PL-}\Phi}(\me_t|\vec\theta^*)} {\phi^{s_t}_{\ma_t}}-\frac {\#\text{ of }\me_t} {n\phi^{s_t}_{\ma_t}}]=\frac 1 {\phi^{s_t}_{\ma_t}}(\Pr\nolimits_{k\text{-PL-}\Phi}(\me_t|\vec\theta^*)-E[\frac {\#\text{ of }\me_t} {n}])=\vec 0.$$

Let $Q_n(\vec\theta)=||g(P, \vec\theta)||^2_2$, which is minimized at $\vec\theta'$ (the estimate) and define $Q_0(\vec\theta)=E[Q_n(\vec\theta)]$, which is minimized at $\vec\theta^*$. 
We first prove the following lemma:
\begin{lem}\label{lem:uniformconvergence}
$\sup_{\vec\theta\in\Theta} |Q_n(\vec\theta)-Q_0(\vec\theta)|\xrightarrow{p}0$.
\end{lem}
\begin{proof}
Recall that any moment condition $g(O_j, \vec\theta)$ (corresponding to partial order $\me_t$ where $1\le t\le q$) has the from $\Pr_{k\text{-PL-}\Phi}(\me_t|\vec\theta)-X_{t, j}$ where $X_{t, j}=1$ if $\me_t$ is observed from $O_j$ and $X_{t, j}=0$ otherwise. And also from $\vec g_n(\vec\theta)=\vec g(P, \vec\theta)=\frac 1 n\sum^n_{j=1}\vec g(O_j, \vec\theta)$, for any moment condition, we have
$$|g_n(\vec\theta)-g_0(\vec\theta)|=|\frac 1 n\sum^n_{j=1} X_{t, j}-E[X_t]|\xrightarrow{p}0.$$
Therefore, we obtain $\sup_{\vec\theta\in\Theta}||\vec g_n(\vec\theta)-\vec g_0(\vec\theta)||\xrightarrow{p}\vec 0$. 

Then we have (omitting the independent variable $\vec\theta$)
$$
|Q_n-Q_0|=|\vec g^\top_n\vec g_n-\vec g^\top_0\vec g_0|
\le|(\vec g_n-\vec g_0)^\top(\vec g_n-\vec g_0)|+2|\vec g^\top_0(\vec g_n-\vec g_0)|
$$
Since all moment conditions fall in $[-1, 1]$ for any $\vec\theta\in\Theta$, we have 
$$\sup_{\vec\theta\in\Theta}|Q_n(\vec\theta)-Q_0(\vec\theta)|\xrightarrow{p}0.$$
\end{proof}

Now we are ready to prove consistency. By our Theorem~\ref{thm:id}, the model is identifiable, which means $g_0(\vec\theta)$ is uniquely minimized at $\vec\theta^*$. Since $Q_0(\vec\theta)$ is continuous and $\Theta$ is compact ($\theta^{(r)}_i\in [\epsilon, 1]$ for all $r=0, 1$ and $i=1, \ldots, m$), by Lemma~\ref{lem:uniformconvergence} and Theorem 2.1 by \citet{Newey94:Large}, we have $\vec\theta'\xrightarrow{p}\vec\theta^*$.
\end{proof}
%\noindent{\bf Asymptotic Normality.} Given consistency, standard proof of asymptotic normality for generalized method of moments applies to our case. It's not hard to check all the conditions of Theorem 3.4 by \citet{Newey94:Large} hold as follows (letting $W$ be the identity matrix). Since (i) $\vec\theta^*$ lies in the interior of $\Theta$; (ii) $\vec g(R, \vec\theta)$ is continuously differentiable; (iii) $E[\vec g(R, \vec\theta)]=\vec 0$ and $E[||\vec g(R, \vec\theta)||^2]$ is finite; (iv) $E[\sup_{\vec\theta\in\Theta}||\nabla\vec g(R, \vec\theta)||]<\infty$; (v) $G^\top G$ is nonsingular where $G=E[
%\nabla \vec g(R, \vec\theta^*)]$ because we have more independent moment conditions than the number of parameters. Then by Theorem 3.4 of \cite{Newey94:Large}, we have
%$$\sqrt{n}(\vec\theta'-\vec\theta^*)\xrightarrow{d}N(\vec 0, (G^\top G)G^\top\Omega G(G^\top G)^{-1}),$$
%where $\Omega=E[\vec g(R, \vec\theta^*)\vec g^\top(R, \vec\theta^*)]$

\section{Experiments}\label{sec:exp}

\noindent{\bf Setup.} We conducted experiments on synthetic data to demonstrate the effectiveness of our algorithms. The data are generated as follows: (i) generate $\alpha$, $\vec\theta^{(1)}$, and $\vec\theta^{(2)}$ uniformly at random and normalize s.t. $\sum^m_{i=1}\theta^{(r)}_i=1$ for $r=1, 2$; (ii) generate linear orders using $k$-PL-linear; (iii) choose $\ptopl{l}$, $\plway{l}{\ma'}$, and $\pchoice{l}{\ma'}$ and sample partial orders from the generated linear orders. The partial orders are generated from the following two models:
\begin{itemize}
\item ranked top-$2$ and $2$-way: $\ptopl{2}=\frac 1 2$, $\plway{2}{\ma'}=\frac {1} {m(m-1)}$ for all $\ma'\subset\ma$ and $|\ma'|=2$;
\item choice-$2,3,4$: first group the alternatives as described in the previous section. Let $C=\lceil\frac {m-1} 3\rceil$ be the number of groups. We first sample a group uniformly at random. Let $\ma^{(4)}$ be the sampled group (of four alternatives). Then $\pchoice{4}{\ma^{(4)}}=\frac 1 C\frac 4 {28}$; for each subset $\ma^{(3)}\subset\ma^{(4)}$ of three alternatives (four such subsets within $\ma^{(4)}$), $\pchoice{3}{\ma^{(3)}}=\frac 1 C\frac 3 {28}$; for each subset $\ma^{(2)}\subset\ma^{(4)}$ of two alternatives (six subsets within $\ma^{(4)}$), $\pchoice{2}{\ma^{(2)}}=\frac 1 C\frac 1 {28}$.
\end{itemize}

Besides, we tested our algorithms on linear orders. In this case, all partial orders are marginal events of linear orders and there is no $\vec\phi$ estimation. Our algorithms reduce to the standard generalized-method-of-moments algorithms. 

The baseline algorithms are the GMM algorithm by~\cite{Zhao16:Learning} and ELSR-Gibbs algorithm by~\cite{Liu19:Learning}. The GMM algorithm by~\cite{Zhao16:Learning} is for linear order, but it utilizes only ranked top-3 orders. So it can be viewed as both a linear order algorithm and a partial order algorithm. We apply ELSR-Gibbs algorithm by~\cite{Liu19:Learning} on ``choice-2,3,4" datasets because the algorithm is expected to run faster than ``ranked top-2 and 2-way" dataset.

All algorithms were implemented with MATLAB\footnote{Code available at \url{https://github.com/zhaozb08/MixPL-SPO}} on an Ubuntu Linux server with Intel Xeon E5 v3 CPUs each clocked at 3.50 GHz. We use Mean Squared Error (MSE), which is defined as $E[||\vec\theta'-\vec\theta^*||^2_2]$, and runtime to compare the performance of the algorithms. For fair comparisons with previous works, we ignore the $\vec\phi$ parameter when computing MSE. 

\begin{figure*}[!ht]
\centering
\includegraphics[width = 0.5\textwidth]{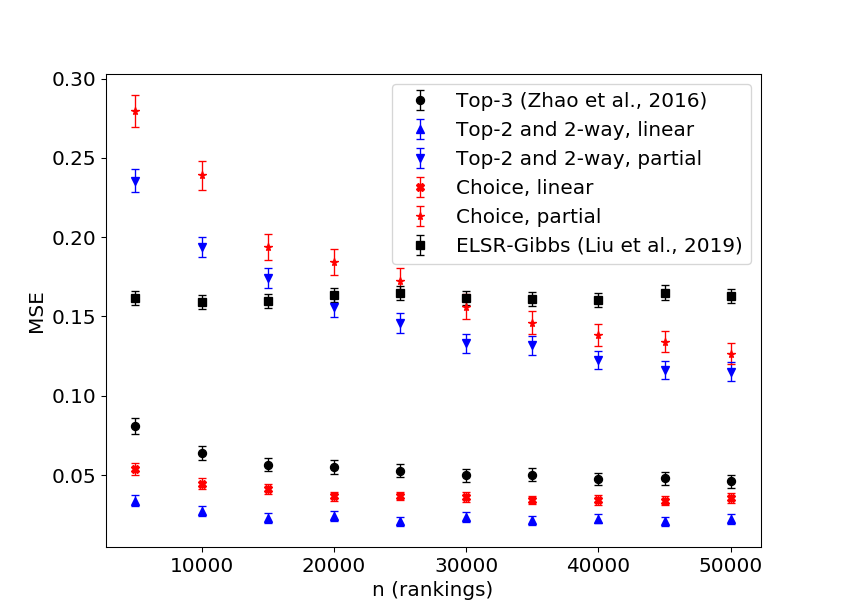}\includegraphics[width = 0.5\textwidth]{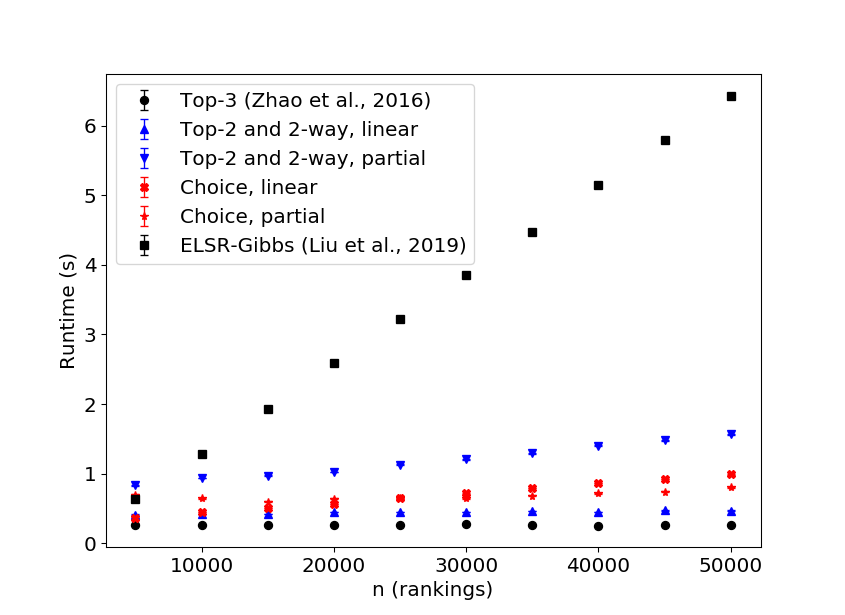}
\caption{MSE and runtime with 95\% confidence intervals for 2-PL over 10 alternatives when $n$ varies. ``Choice" denotes the setting of ``choice-$2,3,4$". For ELSR-Gibbs~\cite{Liu19:Learning}, we used the partial orders generated by ``choice-$2,3,4$". One linear extension was generated from each partial order and three EM iterations were run. All values were averaged over $2000$ trials.}
\label{fig:k2m10}
\end{figure*}

\noindent{\bf Results and Discussions.} The algorithms are compared when the number of rankings varies (Figure~\ref{fig:k2m10}). We have the following observations.
\begin{itemize}
\item When learning from partial orders only: ``ELSR-gibbs~\cite{Liu19:Learning}" is much slower than other algorithms for large datasets. MSEs of all other algorithms converge towards zero as $n$ increases. We can see ``top-2 and 2-way, partial" and ``choice, partial" converge slower than ``top-3". Ranked top-$l$ orders are generally more informative for parameter estimation than other partial orders. However, as was reported in \cite{Zhao18:Cost}, it is much more time consuming for human to pick their ranked top alternative(s) from a large set of alternatives than fully rank a small set of alternatives, which means ranked top-$l$ data are harder or more costly to collect.
\item When learning from linear orders: our ``ranked top-2 and 2-way, linear" and ``choice-$2,3,4$, linear" outperform ``top-3~\cite{Zhao16:Learning}" in terms of MSE (left of Figure~\ref{fig:k2m10}), but only slightly slower than ``top-3~\cite{Zhao16:Learning}" (Figure~\ref{fig:k2m10} right).
\end{itemize}
%\begin{figure*}[!ht]
%\centering
%\includegraphics[width = %0.5\textwidth]{figures/mse_k2n2000.png}\includegraphics[width = 0.5\textwidth]{figures/runtime_k2n2000.png}
%\caption{MSE and runtime for 2-PL when $m$ varies. $n$ is set to be $2000$ for all algorithms. ``Choice" denotes the setting of ``choice-$2,3,4$". Values were averaged over $5000$ trials.}
%\label{fig:k2n2000}
%\end{figure*}

\section{Conclusions and Future Work}

We extend the mixtures of Plackett-Luce models to the class of models that sample structured partial orders and theoretically characterize the (non-)identifiability of this class of models. We propose consistent and efficient algorithms to learn mixtures of two Plackett-Luce models from linear orders or structured partial orders. For future work, we will explore more statistically and computationally efficient algorithms for mixtures of an arbitrary number of Plackett-Luce models, or the more general random utility models.

\subsubsection*{Acknowledgments}

We thank all anonymous reviewers for helpful comments and suggestions. This work is supported by NSF \#1453542 and ONR \#N00014-17-1-2621.

\small

\bibliography{references.bib}

\begin{thebibliography}{35}
\providecommand{\natexlab}[1]{#1}
\providecommand{\url}[1]{\texttt{#1}}
\expandafter\ifx\csname urlstyle\endcsname\relax
  \providecommand{\doi}[1]{doi: #1}\else
  \providecommand{\doi}{doi: \begingroup \urlstyle{rm}\Url}\fi

\bibitem[Altman and Tennenholtz(2005)]{Altman05:PageRank}
Alon Altman and Moshe Tennenholtz.
\newblock Ranking systems: The {PageRank} axioms.
\newblock In \emph{Proceedings of the ACM Conference on Electronic Commerce
  (EC)}, Vancouver, BC, Canada, 2005.

\bibitem[Ammar et~al.(2014)Ammar, Oh, Shah, and Voloch]{Ammar14:What}
Ammar Ammar, Sewoong Oh, Devavrat Shah, and L~Voloch.
\newblock What's your choice? learning the mixed multi-nomial logit model.
\newblock In \emph{Proceedings of the ACM SIGMETRICS/international conference
  on Measurement and modeling of computer systems}, 2014.

\bibitem[Baltrunas et~al.(2010)Baltrunas, Makcinskas, and
  Ricci]{Baltrunas10:Group}
Linas Baltrunas, Tadas Makcinskas, and Francesco Ricci.
\newblock Group recommendations with rank aggregation and collaborative
  filtering.
\newblock In \emph{Proceedings of the fourth ACM conference on Recommender
  systems}, pages 119--126. ACM, 2010.

\bibitem[Brandt and Geist(2015)]{Brandt2015:Pnyx}
Felix Brandt and Guillaume Chabinand~Christian Geist.
\newblock {Pnyx:: A Powerful and User-friendly Tool for Preference
  Aggregation}.
\newblock In \emph{Proceedings of the 2015 International Conference on
  Autonomous Agents and Multiagent Systems}, pages 1915--1916, 2015.

\bibitem[Cand{\`e}s and Recht(2009)]{Candes09:Exact}
Emmanuel~J Cand{\`e}s and Benjamin Recht.
\newblock Exact matrix completion via convex optimization.
\newblock \emph{Foundations of Computational mathematics}, 9\penalty0
  (6):\penalty0 717, 2009.

\bibitem[Chen et~al.(2013)Chen, Bennett, Collins-Thompson, and
  Horvitz]{Chen13:Pairwise}
Xi~Chen, Paul~N Bennett, Kevyn Collins-Thompson, and Eric Horvitz.
\newblock Pairwise ranking aggregation in a crowdsourced setting.
\newblock In \emph{Proceedings of the sixth ACM international conference on Web
  search and data mining}, pages 193--202. ACM, 2013.

\bibitem[Chierichetti et~al.(2018)Chierichetti, Kumar, and
  Tomkins]{Chierichetti18:Learning}
Flavio Chierichetti, Ravi Kumar, and Andrew Tomkins.
\newblock Learning a mixture of two multinomial logits.
\newblock In \emph{Proceedings of the 35rd International Conference on Machine
  Learning (ICML-18)}, 2018.

\bibitem[Gormley and Murphy(2008)]{Gormley08:Exploring}
Isobel~Claire Gormley and Thomas~Brendan Murphy.
\newblock Exploring voting blocs within the irish exploring voting blocs within
  the irish electorate: {A} mixture modeling approach.
\newblock \emph{Journal of the American Statistical Association}, 103\penalty0
  (483):\penalty0 1014--1027, 2008.

\bibitem[Gormley and Murphy(2009)]{Gormley09:Grade}
Isobel~Claire Gormley and Thomas~Brendan Murphy.
\newblock A grade of membership model for rank data.
\newblock \emph{Bayesian Analysis}, 4\penalty0 (2):\penalty0 265--296, 2009.

\bibitem[Huang et~al.(2011)Huang, Kapoor, and Guestrin]{Huang11:Efficient}
Jonathan Huang, Ashish Kapoor, and Carlos Guestrin.
\newblock Efficient probabilistic inference with partial ranking queries.
\newblock In \emph{Proceedings of the Twenty-Seventh Conference on Uncertainty
  in Artificial Intelligence}, pages 355--362. AUAI Press, 2011.

\bibitem[Hunter(2004)]{Hunter04:MM}
David~R. Hunter.
\newblock {MM algorithms for generalized Bradley-Terry models}.
\newblock In \emph{The Annals of Statistics}, volume~32, pages 384--406, 2004.

\bibitem[Jamieson and Nowak(2011)]{Jamieson11:Active}
Kevin~G Jamieson and Robert Nowak.
\newblock Active ranking using pairwise comparisons.
\newblock In \emph{Advances in Neural Information Processing Systems}, pages
  2240--2248, 2011.

\bibitem[Jang et~al.(2016)Jang, Kim, Suh, and Oh]{Jang16:Top}
Minje Jang, Sunghyun Kim, Changho Suh, and Sewoong Oh.
\newblock Top-$ k $ ranking from pairwise comparisons: When spectral ranking is
  optimal.
\newblock \emph{arXiv preprint arXiv:1603.04153}, 2016.

\bibitem[Keshavan et~al.(2010)Keshavan, Montanari, and Oh]{Keshavan10:Matrix}
Raghunandan~H Keshavan, Andrea Montanari, and Sewoong Oh.
\newblock Matrix completion from noisy entries.
\newblock \emph{Journal of Machine Learning Research}, 11\penalty0
  (Jul):\penalty0 2057--2078, 2010.

\bibitem[Khetan and Oh(2016)]{Khetan16:Data}
Ashish Khetan and Sewoong Oh.
\newblock Data-driven rank breaking for efficient rank aggregation.
\newblock \emph{Journal of Machine Learning Research}, 17\penalty0
  (193):\penalty0 1--54, 2016.

\bibitem[Liu et~al.(2019)Liu, Zhao, Liao, Lu, and Xia]{Liu19:Learning}
Ao~Liu, Zhibing Zhao, Chao Liao, Pinyan Lu, and Lirong Xia.
\newblock Learning plackett-luce mixtures from partial preferences.
\newblock In \emph{Proceedings of the Thirty-Third AAAI Conference on
  Artificial Intelligence (AAAI-19)}, 2019.

\bibitem[Liu(2011)]{Liu11:Learning}
Tie-Yan Liu.
\newblock \emph{Learning to Rank for Information Retrieval}.
\newblock Springer, 2011.

\bibitem[Lu and Boutilier(2014)]{Lu14:Effective}
Tyler Lu and Craig Boutilier.
\newblock Effective sampling and learning for mallows models with
  pairwise-preference data.
\newblock \emph{The Journal of Machine Learning Research}, 15\penalty0
  (1):\penalty0 3783--3829, 2014.

\bibitem[Luce(1959)]{Luce59:Individual}
Robert~Duncan Luce.
\newblock \emph{{Individual Choice Behavior: A Theoretical Analysis}}.
\newblock Wiley, 1959.

\bibitem[Mao et~al.(2013)Mao, Procaccia, and Chen]{Mao13:Better}
Andrew Mao, Ariel~D. Procaccia, and Yiling Chen.
\newblock Better human computation through principled voting.
\newblock In \emph{Proceedings of the National Conference on Artificial
  Intelligence (AAAI)}, Bellevue, WA, USA, 2013.

\bibitem[Marden(1995)]{Marden95:Analyzing}
John~I. Marden.
\newblock \emph{Analyzing and modeling rank data}.
\newblock Chapman \& Hall, 1995.

\bibitem[Maystre and Grossglauser(2015)]{Maystre15:Fast}
Lucas Maystre and Matthias Grossglauser.
\newblock Fast and accurate inference of plackett--luce models.
\newblock In \emph{Advances in neural information processing systems}, pages
  172--180, 2015.

\bibitem[Mollica and Tardella(2017)]{Mollica17:Bayesian}
Cristina Mollica and Luca Tardella.
\newblock Bayesian {P}lackett--{L}uce mixture models for partially ranked data.
\newblock \emph{Psychometrika}, 82\penalty0 (2):\penalty0 442--458, 2017.

\bibitem[Negahban and Wainwright(2012)]{Negahban12:Restricted}
Sahand Negahban and Martin~J Wainwright.
\newblock Restricted strong convexity and weighted matrix completion: Optimal
  bounds with noise.
\newblock \emph{Journal of Machine Learning Research}, 13\penalty0
  (May):\penalty0 1665--1697, 2012.

\bibitem[Newey and McFadden(1994)]{Newey94:Large}
Whitney~K Newey and Daniel McFadden.
\newblock Large sample estimation and hypothesis testing.
\newblock \emph{Handbook of econometrics}, 4:\penalty0 2111--2245, 1994.

\bibitem[Oh and Shah(2014)]{Oh14:Learning}
Sewoong Oh and Devavrat Shah.
\newblock Learning mixed multinomial logit model from ordinal data.
\newblock In \emph{Advances in Neural Information Processing Systems}, pages
  595--603, 2014.

\bibitem[Pini et~al.(2011)Pini, Rossi, Venable, and
  Walsh]{Pini11:Incompleteness}
Maria~Silvia Pini, Francesca Rossi, Kristen~Brent Venable, and Toby Walsh.
\newblock {Incompleteness and incomparability in preference aggregation:
  Complexity results}.
\newblock \emph{Artificial Intelligence}, 175\penalty0 (7--8):\penalty0
  1272---1289, 2011.

\bibitem[Plackett(1975)]{Plackett75:Analysis}
Robin~L. Plackett.
\newblock The analysis of permutations.
\newblock \emph{Journal of the Royal Statistical Society. Series C (Applied
  Statistics)}, 24\penalty0 (2):\penalty0 193--202, 1975.

\bibitem[Redner and Walker(1984)]{Redner84:Mixture}
Richard~A Redner and Homer~F Walker.
\newblock Mixture densities, maximum likelihood and the em algorithm.
\newblock \emph{SIAM review}, 26\penalty0 (2):\penalty0 195--239, 1984.

\bibitem[Tkachenko and Lauw(2016)]{Tkachenko16:Plackett}
Maksim Tkachenko and Hady~W Lauw.
\newblock Plackett-luce regression mixture model for heterogeneous rankings.
\newblock In \emph{Proceedings of the 25th ACM International on Conference on
  Information and Knowledge Management}, pages 237--246. ACM, 2016.

\bibitem[Train(2009)]{Train09:Discrete}
Kenneth~E. Train.
\newblock \emph{{Discrete Choice Methods with Simulation}}.
\newblock Cambridge University Press, 2nd edition, 2009.

\bibitem[Xia(2019)]{Xia2019:Learning}
Lirong Xia.
\newblock \emph{{Learning and Decision-Making from Rank Data}}.
\newblock Synthesis Lectures on Artificial Intelligence and Machine Learning.
  Morgan \& Claypool Publishers, 2019.

\bibitem[Zhao et~al.(2016)Zhao, Piech, and Xia]{Zhao16:Learning}
Zhibing Zhao, Peter Piech, and Lirong Xia.
\newblock Learning mixtures of {P}lackett-{L}uce models.
\newblock In \emph{Proceedings of the 33rd International Conference on Machine
  Learning (ICML-16)}, 2016.

\bibitem[Zhao et~al.(2018{\natexlab{a}})Zhao, Li, Wang, Kephart, Mattei, Su,
  and Xia]{Zhao18:Cost}
Zhibing Zhao, Haoming Li, Junming Wang, Jeffrey Kephart, Nicholas Mattei, Hui
  Su, and Lirong Xia.
\newblock A cost-effective framework for preference elicitation and
  aggregation.
\newblock In \emph{Proceedings of the 34th Conference on Uncertainty in
  Artificial Intelligence (UAI-2018)}, 2018{\natexlab{a}}.

\bibitem[Zhao et~al.(2018{\natexlab{b}})Zhao, Villamil, and
  Xia]{Zhao18:Learning}
Zhibing Zhao, Tristan Villamil, and Lirong Xia.
\newblock Learning mixtures of random utility models.
\newblock In \emph{Proceedings of the Thirty-Second AAAI Conference on
  Artificial Intelligence (AAAI-18)}, 2018{\natexlab{b}}.

\end{thebibliography}
\bibliographystyle{plainnat}

\iffalse

\section{Experiments on Real-World Data}

We learn $k$-PLs from movie and sushi datasets on Preflib~\cite{Mattei13:Preflib} using E-LSR algorithm~\cite{Liu19:Learning} and compute their AIC and BIC, defined as: $\text{AIC} = 2d-2\ln(L)$ and $\text{BIC} = d\ln(n)-2\ln(L)$, where $L$ is the value of the likelihood function evaluated at the estimation, $d$ is the number of parameters in the model, and $n$ is the number of rankings. A smaller AIC or BIC means better fitness. The movie dataset consists of 1814 rankings over four alternatives and the sushi dataset 5000 rankings over ten alternatives.

\begin{figure}[!ht]
\centering
\includegraphics[width=0.5\textwidth]{figures/aicbic_movie}\includegraphics[width=0.5\textwidth]{figures/aicbic_sushi}
\caption{Model fitness when $k$ increases on movie data.}
\label{fig:movie}
\end{figure}

{\bf We observe a significant improvement in terms of AIC and BIC from $k=1$ to $k=2$.} For the movie dataset (Figure~\ref{fig:movie} left) where there are four alternatives, the best $k$ is $3$ in terms of BIC and $4$ in terms of AIC. For the sushi dataset (Figure~\ref{fig:movie} right), the best $k$ in terms of BIC is $16$ and in terms of AIC is over $30$.

\fi

\end{document}